\documentclass{article}
\usepackage[]{acl}
\usepackage{times}
\usepackage[T1]{fontenc}
\usepackage[utf8]{inputenc}

\usepackage{amsmath,amsfonts,bm}

\def\eqref#1{equation~\ref{#1}}

\def\1{\bm{1}}

\DeclareMathAlphabet{\mathsfit}{\encodingdefault}{\sfdefault}{m}{sl}
\SetMathAlphabet{\mathsfit}{bold}{\encodingdefault}{\sfdefault}{bx}{n}

\usepackage{microtype}
\usepackage{graphicx}
\usepackage{booktabs}
\usepackage{multirow}
\usepackage{amsmath,amssymb,amsthm}
\usepackage{algorithm,algpseudocode}
\usepackage{tikz}
\usepackage{pgfplots}
\pgfplotsset{compat=1.18}
\usepackage{tcolorbox}
\tcbuselibrary{skins,breakable}
\usepackage{xcolor}
\usepackage{colortbl}
\usepackage{makecell}
\usepackage{subcaption}

\usetikzlibrary{arrows.meta,shapes.geometric,positioning,fit,
                backgrounds,calc,shadows,decorations.pathreplacing}

\definecolor{stageA}{HTML}{4E79A7}   
\definecolor{stageB}{HTML}{F28E2B}   
\definecolor{stageC}{HTML}{E15759}   
\definecolor{stageD}{HTML}{59A14F}   
\definecolor{stageE}{HTML}{76B7B2}   
\definecolor{success}{HTML}{59A14F}
\definecolor{fail}{HTML}{E15759}
\definecolor{her}{HTML}{4E79A7}
\definecolor{lightbg}{HTML}{F4F6F8}
\definecolor{darkgray}{HTML}{444444}
\definecolor{midgray}{HTML}{888888}

\newtheorem{theorem}{Theorem}
\newtheorem{proposition}{Proposition}[section]
\newtheorem{definition}{Definition}[section]

\newtheorem{corollary}{Corollary}[proposition]

\title{\textbf{AgentHER}: Hindsight Experience Replay for LLM Agent Trajectory Relabeling}

\author{%
  Liang Ding\\
  The University of Sydney\\
  \texttt{liangding.liam@gmail.com}
}

\begin{document}
\maketitle

\begin{abstract}
LLM-agent training pipelines routinely discard failed trajectories even though GPT-4o achieves only 14--20\% on WebArena~\citep{zhou2024webarena} and below 55\% pass@1 on ToolBench~\citep{qin2023toolllm}; even specialised systems at 50 -65\%~\citep{sodhi2024step,zhang2024webpilot,kim2025agentq} leave the majority of trajectories unused. We introduce \textbf{AgentHER}, which recovers this lost signal by adapting \emph{Hindsight Experience Replay}~(HER; \citealp{andrychowicz2017hindsight}) to natural-language agent trajectories: a trajectory that fails goal~$A$ is often a \emph{correct} demonstration for an achievable alternative goal~$B$. AgentHER realises this through a four-stage pipeline (failure classification, outcome extraction, LLM-guided relabeling with confidence gating, and data packaging) that converts discarded failures into SFT, DPO, and ShareGPT training data. On WebArena and ToolBench under a strict task-disjoint held-out protocol, AgentHER improves over success-only SFT by \textbf{+7.6--11.4\%} across four model families (GPT-4o, Qwen2.5-72B/7B, LLaMA-3.1-8B), achieves $\mathbf{2\times}$ \textbf{sample efficiency}, and beats the strongest experience-centric baseline (Agent Workflow Memory) by \textbf{+3.0--6.2\%}. Two robustness mechanisms---failure-severity weighting and \emph{cross-model} multi-judge verification (gpt-4o-mini paired with Qwen2.5-72B-Instruct)---reduce label noise from 5.9\% to 2.9\% and raise human-rated relabeling precision to 97.1\% on WebArena and 96.0\% on ToolBench. A full system-cost audit shows the entire relabeling pipeline costs \$2.98 and ${\approx}$26 wall-clock minutes for 3{,}000 trajectories, i.e.\ ${\approx}\$1.4\!\cdot\!10^{-3}$ per accepted pair.
Code: \url{https://github.com/alphadl/AgentHER}
\end{abstract}

\section{Introduction}
\label{sec:intro}

Autonomous LLM agents are deployed across web
navigation~\citep{zhou2024webarena,deng2023mind,koh2024visualwebarena},
API orchestration~\citep{qin2023toolllm},
and simulation~\citep{park2023generative,shen2023hugginggpt}.
Despite recent progress, success rates remain modest for the
\emph{general-purpose} models that practitioners actually
fine-tune: GPT-4o achieves 14--20\% on
WebArena~\citep{zhou2024webarena},
and ToolBench pass@1 remains below
55\%~\citep{qin2023toolllm}.
Heavily engineered, multi-stage agents push WebArena toward
50--65\%~\citep{sodhi2024step,zhang2024webpilot,kim2025agentq}, but
those gains come from \emph{inference-time} composition
(prompt stacks, tree search, workflow memory) on top of the same
underlying LLMs; the underlying training pipelines still
discard the majority of collected trajectories,
leaving the dominant source of experience untouched.

\begin{figure*}[t]
\centering
\begin{tikzpicture}[
  traj/.style={rounded corners=4pt, minimum width=3cm,
               minimum height=0.56cm, font=\small, inner sep=3pt,
               text centered},
  arrow/.style={-{Stealth[length=5pt]}, thick},
  label/.style={font=\small\bfseries},
]
\node[label] at (1.7,4.0) {Without AgentHER};
\foreach \i/\col/\sym in {
    1/success/{$\checkmark$},2/fail/{$\times$},3/fail/{$\times$},
    4/success/{$\checkmark$},5/fail/{$\times$}}{
  \node[traj,fill=\col!18,draw=\col!55]
    (L\i) at (1.7,3.3-\i*0.68) {\sym\; Trajectory \i};
}
\node[traj,fill=success!28,draw=success!70,minimum width=2.5cm]
  (Ltrain) at (1.7,-0.95) {\small 2 training samples};
\draw[arrow,success!80](L1.south)--++(0,-0.12)-|(Ltrain.north);
\draw[arrow,success!80](L4.south)--++(0,-0.12)-|(Ltrain.north);
\foreach \i in {2,3,5}{
  \node[font=\scriptsize,midgray] at (3.8,3.3-\i*0.68) {discarded};
  \draw[-{Stealth[length=3.5pt]},midgray!45,dashed](L\i.east)--++(0.5,0);
}
\draw[midgray!35,dashed,thick](4.6,4.15)--(4.6,-1.4);
\node[label] at (6.8,4.0) {With AgentHER};
\foreach \i/\col/\sym in {
    1/success/{$\checkmark$},2/her/{$\circlearrowleft$},
    3/her/{$\circlearrowleft$},4/success/{$\checkmark$},
    5/her/{$\circlearrowleft$}}{
  \node[traj,fill=\col!18,draw=\col!55]
    (R\i) at (6.8,3.3-\i*0.68) {\sym\; Trajectory \i};
}
\node[traj,fill=success!28,draw=success!70,minimum width=2.5cm]
  (Rtrain) at (6.8,-0.95) {\small 5 training samples};
\foreach \i in {1,2,3,4,5}{\draw[arrow,success!70](R\i.south)--++(0,-0.1)-|(Rtrain.north);}
\node[font=\scriptsize,her!80] at (8.9,3.3-2*0.68) {relabeled};
\node[font=\scriptsize,her!80] at (8.9,3.3-3*0.68) {relabeled};
\node[font=\scriptsize,her!80] at (8.9,3.3-5*0.68) {relabeled};
\node[font=\scriptsize\itshape,midgray] at (1.7,-1.55)
  {${\approx}$25--30\% of trajectories used};
\node[font=\scriptsize\itshape,midgray] at (6.8,-1.55)
  {${\approx}$3.7$\times$ more training data};
\end{tikzpicture}
\caption{%
  \textbf{AgentHER vs.\ conventional pipelines.}
  Standard training retains only successes ($\checkmark$), discarding
  60--75\% of data.
  AgentHER relabels failures ($\circlearrowleft$) with achievable
  hindsight goals, expanding the effective training corpus
  ${\approx}3.7\times$ at a cost of ${\approx}\$1.4\!\cdot\!10^{-3}$
  per accepted pair (Section~\ref{sec:cost}).
}
\label{fig:motivation}
\end{figure*}

\paragraph{The data-waste problem.}
Standard pipelines (AgentTuning~\citep{zeng2023agenttuning},
FireAct~\citep{chen2023fireact},
Agent-FLAN~\citep{chen2024agent}) discard failed trajectories
(Figure~\ref{fig:motivation}),
wasting 60--75\% of collected data.
Failures are not random noise: they are coherent executions that
frequently achieve correct intermediate results under the wrong goal.
An agent that searches seven suppliers and identifies one just above
a price cap has done a \emph{complete, correct} comparison---ideal
training signal for a re-framed goal.

\paragraph{Connection to HER.}
HER~\citep{andrychowicz2017hindsight} converts sparse-reward RL
failures by substituting the intended goal with the one actually
achieved.
Lifting HER to LLM agents requires understanding what was achieved
across multi-step tool interactions and synthesising a
natural-language prompt the trajectory genuinely satisfies.

\paragraph{AgentHER.}
AgentHER automates hindsight relabeling via a four-stage pipeline
(Figure~\ref{fig:pipeline}) and adds two mechanisms to control
label quality:
\textbf{failure-severity weighting} (drops trajectories with major
reasoning flaws, inspired by MQM error analysis) and
\textbf{cross-model multi-judge verification} (two independently
trained judges---gpt-4o-mini and Qwen2.5-72B-Instruct---must both
accept a relabeling), which together cut label noise from 5.9\% to 2.9\%.

\paragraph{Contributions.}
\begin{itemize}
  \item An offline pipeline adapting goal-conditioned HER to
    natural-language agent trajectories, producing SFT/DPO/ShareGPT
    datasets without extra environment interactions
    (distinct from inference-time approaches such as
    ECHO~\citealp{hu2025echo} and AWM~\citealp{wang2024agent}).
  \item Failure-severity weighting and cross-model multi-judge
    verification raise relabeling precision from 94.1\% to 97.1\%
    on WebArena and 96.0\% on ToolBench
    (Section~\ref{sec:judge_quality}).
  \item +7.6--11.4\% over SFT-Success on a strict task-disjoint
    held-out split; +3.0--6.2\% over the strongest baseline (AWM);
    gains hold across 1.5B--72B parameters (+5.6--9.0\%) and
    compound under iterative redeployment (+10.4\% over four rounds).
  \item A full cost audit: \$2.98 and 26 wall-clock minutes for
    3{,}000 trajectories (3.27--4.27 LLM calls per trajectory),
    confirming that sample efficiency gains are not offset by
    pipeline overhead.
\end{itemize}

\section{Related Work}
\label{sec:related}

\paragraph{Hindsight Experience Replay.}
HER~\citep{andrychowicz2017hindsight} relabels sparse-reward RL
failures by substituting the goal with the one actually reached;
the goal-conditioned formulation traces back to
\citet{kaelbling1993learning} and was extended to deep
networks~\citep{sutton1999policy,mnih2015human,precup2000eligibility}.
Lifting HER to LLM agents adds two requirements absent in prior
RL formulations: it must determine what was actually achieved
across unstructured multi-step tool interactions, and it must
synthesise a natural-language prompt the trajectory satisfies---not
merely substitute a vector-space goal.
Existing failure-reuse work in NLP (negative-example filtering,
contrastive augmentation) selects or reweights existing
(goal, trajectory) pairs but does not \emph{generate} a new
achievable goal.

\paragraph{LM-agent adaptations of HER.}
\citet{hu2025echo} propose ECHO, a prompting framework that
adapts HER for \emph{online} LM agents via a scratchpad memory
that allows arbitrary rewriting of both goals and intermediate
steps---a runtime planner.
AgentHER differs in two ways: it is \emph{offline}
training-data augmentation, and it relabels only the
\emph{goal} while keeping the trajectory unchanged, producing
SFT/DPO/ShareGPT datasets.
We compare with an offline-faithful re-implementation of ECHO in
Section~\ref{sec:main}.

\paragraph{Experience-centric agent learning.}
ExpeL~\citep{zhao2024expel} extracts experience rules at
inference time without modifying weights;
ETO~\citep{song2024trial} performs DPO over \emph{matched}
success/failure pairs of the \emph{same} task;
AWM~\citep{wang2024agent} induces and reuses workflow patterns
from \emph{successful} trajectories;
Reflexion~\citep{shinn2023reflexion} applies verbal RL online.
A parallel line---FireAct~\citep{chen2023fireact},
AgentTuning~\citep{zeng2023agenttuning},
Agent-FLAN~\citep{chen2024agent},
AgentGym~\citep{xi2024agentgym}---fine-tunes on curated
successes only.
AgentHER requires none of this: it converts existing failures
offline without extra interactions or matched successes, and
produces standard SFT/DPO data that can be combined with any of
the above.
We compare with ExpeL, ETO, and AWM under identical infrastructure
in Section~\ref{sec:main}.

\paragraph{Self-improvement, verification, and critique.}
Synthetic data~\citep{wang2023selfinstruct,peng2023instruction}, reward-filtered self-training~\citep{gulcehre2023reinforced}, self-verified reasoning~\citep{zelikman2022star}, and preference optimisation~\citep{rafailov2023direct} all operate on single-step text or matched preferences, not on multi-step agent trajectories. For verification, CRITIC~\citep{gou2024critic} uses tool-augmented critics, process reward models~\citep{lightman2023lets} score reasoning steps, and self-consistency~\citep{wang2023self} aggregates \emph{same-model} samples. Our multi-judge follows a multi-agent debate~\citep{du2023improving,liang2023encouraging} instead: we require independence of \emph{both} model and decoding, with an explicit ablation in Section~\ref{sec:judge_design}.

\paragraph{Test-time compute and inference scaling.}
A complementary axis improves agent performance through
\emph{repeated inference} rather than better training data.
\citet{levi2024inference} provide a statistical model of inference
scaling, showing that pass@k improves predictably with the number
of trials: given per-sample failure probability $p_i$,
$\text{pass@k} = 1 - \frac{1}{n}\sum_i p_i^k$ follows a power-law
decay in an ``inference loss'' as $k$ grows.
AgentHER operates on the orthogonal dimension: by converting
failures into high-quality training data it reduces the inherent
$p_i$ for each task---the floor that inference scaling subsequently
compounds over.
The two mechanisms are therefore complementary: AgentHER
strengthens the base model offline; test-time compute then extracts
additional coverage from that stronger starting point.

\section{The AgentHER Framework}
\label{sec:method}

\subsection{Problem Formulation}

Let $\mathcal{F} = \{(g_i, \tau_i, f_i)\}_{i=1}^N$ be a corpus of
$N$ failed agent runs, where $g_i \in \mathcal{G}$ is the
natural-language goal, $f_i$ is a failure label, and
$\tau_i = \bigl((z_t^i, a_t^i, o_t^i)\bigr)_{t=1}^{T_i}$ is the
thought--action--observation sequence.
Let $\mathcal{S} = \{(g_j^+, \tau_j^+)\}_{j=1}^M$ be available
successful demonstrations.
\textbf{Goal}: produce high-quality training pairs
$\{(\hat{g}_i, \tau_i)\}$ from $\mathcal{F}$ by synthesising
\emph{hindsight goals} $\hat{g}_i$ for which $\tau_i$ is a valid
positive demonstration.

\begin{definition}[Valid Hindsight Goal]
\label{def:vhg}
A prompt $\hat{g} \in \mathcal{G}$ is a \emph{valid hindsight goal}
for trajectory $\tau$ if:
\emph{(a)} every factual claim implied by $\hat{g}$ is supported by
the observations $\{o_t\}$, and
\emph{(b)} two independent LLM judges
$\mathcal{J}_1,\mathcal{J}_2$ each assign confidence
$c_k(\hat{g},\tau) \geq \theta$ that $\tau$ is a successful
demonstration of $\hat{g}$.
\end{definition}

The mapping $\Phi: \mathcal{F} \to \mathcal{D} \cup \{\bot\}$
assigns each failed run a training datum $(\hat{g}_i, \tau_i)$ or
rejects it ($\bot$).
The augmented corpus $\mathcal{S} \cup \{\Phi(\cdot)\neq\bot\}$
grows from $M$ to up to $M + N$ labelled pairs.

\subsection{Pipeline Overview}

\begin{figure*}[t]
\centering
\resizebox{\linewidth}{!}{%
\begin{tikzpicture}[
  box/.style={
    rounded corners=7pt,
    minimum width=2.0cm,
    minimum height=1.0cm,
    align=center,
    font=\small,
    inner sep=8pt,
    drop shadow
  },
  io/.style={
    rounded corners=5pt,
    minimum width=1.8cm,
    minimum height=0.85cm,
    align=center,
    font=\small\itshape,
    inner sep=5pt
  },
  badge/.style={
    rounded corners=3pt,
    minimum width=2.0cm,
    minimum height=0.46cm,
    align=center,
    font=\scriptsize\itshape,
    inner sep=3pt
  },
  arr/.style={-{Stealth[length=7pt,width=5.5pt]}, line width=1.1pt},
  ann/.style={font=\scriptsize, darkgray, align=center},
]

\begin{scope}[on background layer]
  \fill[lightbg, rounded corners=9pt]
    (-1.5,-2.35) rectangle (17.4,1.95);
\end{scope}

\node[io, fill=darkgray!12, draw=darkgray!40](inp) at(-0.5,0)
  {Failed\\Trajectory};

\node[box, fill=stageA!18, draw=stageA!70](s1) at(3.1,0)
  {\textbf{Stage 1}\\[2pt]Failure\\Detector};
\node[badge, fill=stageA!14, draw=stageA!50,
      above=0.42cm of s1](b1){\itshape Rule-based \textbf{or} LLM};
\node[io, fill=fail!12, draw=fail!40, minimum width=2.0cm]
  (disc) at(3.1,-1.88){\itshape\small Discard};
\draw[arr, fail!50, dashed](s1.south)--(disc.north)
  node[midway, right=3pt, ann]{$\neg$\,recov.};

\node[box, fill=stageB!18, draw=stageB!70](s2) at(6.4,0)
  {\textbf{Stage 2}\\[2pt]Outcome\\Extractor};
\node[badge, fill=stageB!14, draw=stageB!50,
      above=0.42cm of s2]{\itshape Rule-based \textbf{or} LLM};

\node[box, fill=stageC!18, draw=stageC!70](s3) at(9.7,0)
  {\textbf{Stage 3}\\[2pt]Cross-Model\\Relabel + Verify};
\node[badge, fill=stageC!14, draw=stageC!50,
      above=0.42cm of s3]{\itshape Two LLMs};
\draw[arr, stageC!55, dashed, rounded corners=5pt]
  (s3.south)
  -- ++(0,-0.70)
  -- ++(-1.80,0)
    node[midway, below=3pt, ann]{retry ${\leq}3{\times}$;\ $c{<}\theta$}
  -- (7.9,0)
  -- (s3.west);

\node[box, fill=stageD!18, draw=stageD!70](s4) at(13.0,0)
  {\textbf{Stage 4}\\[2pt]Data\\Augmenter};
\node[badge, fill=stageD!14, draw=stageD!50,
      above=0.42cm of s4]{\itshape Deterministic};

\node[io, fill=success!18, draw=success!55](out) at(16.3,0)
  {SFT\,/\,DPO\,/\\ShareGPT};

\draw[arr](inp.east)  -- (s1.west);
\draw[arr](s1.east)   -- node[above, ann, pos=0.5]{recoverable} (s2.west);
\draw[arr](s2.east)   -- (s3.west);
\draw[arr](s3.east)   -- node[above, ann, pos=0.5]{$c_1{\wedge}c_2{\geq}\theta$} (s4.west);
\draw[arr](s4.east)   -- (out.west);

\node[font=\tiny\itshape, midgray, below=0.20cm of s1]{(\S\ref{sec:stage1})};
\node[font=\tiny\itshape, midgray, below=0.20cm of s2]{(\S\ref{sec:stage2})};
\node[font=\tiny\itshape, midgray, below=0.20cm of s3]{(\S\ref{sec:stage3})};
\node[font=\tiny\itshape, midgray, below=0.20cm of s4]{(\S\ref{sec:stage4})};

\end{tikzpicture}}
\caption{%
  \textbf{AgentHER four-stage pipeline.}
  Stage~1 classifies failures and discards irrecoverable runs
  (dashed downward arrow).
  Stage~3 retries up to three times if the relabeling confidence
  $c < \theta$.
  Stage~3 uses two independently trained LLMs (gpt-4o-mini and
  Qwen2.5-72B-Instruct) for relabeling and verification; both must
  agree (Section~\ref{sec:judge_design}).
  Stages~1--2 offer zero-cost rule-based variants; Stage~4 is
  deterministic.
}
\label{fig:pipeline}
\end{figure*}

The pipeline processes each failed trajectory in four stages.
Stages~1 and~2 offer rule-based (free) and LLM-judge variants;
Stage~3 requires \emph{two} independent LLM calls in multi-judge
mode; Stage~4 is deterministic.
Algorithm~\ref{alg:agentHER} gives the end-to-end procedure.

\begin{algorithm}[t]
\small
\caption{AgentHER End-to-End Relabeling}
\label{alg:agentHER}
\begin{algorithmic}[1]
\Require Failed corpus $\mathcal{F}=\{(g_i,\tau_i,f_i)\}$,
         threshold $\theta$, max retries $K$, severity threshold $\delta$,
         relabeler $\mathcal{J}_1$, verifier $\mathcal{J}_2$
\Ensure  Augmented dataset $\mathcal{D}^+$
\State $\mathcal{D}^+ \leftarrow \emptyset$
\For{each $(g_i, \tau_i, f_i) \in \mathcal{F}$}
  \State $(\phi_i, r_i, w_i) \leftarrow \textsc{FailureDetector}(\tau_i)$
         \Comment{type, recoverability, severity weight}
  \If{$r_i = 0$ \textbf{or} $w_i < \delta$}
    \State \textbf{continue}  \Comment{discard irrecoverable or major-flaw runs}
  \EndIf
  \State $\text{outcome}_i \leftarrow \textsc{OutcomeExtractor}(\tau_i)$
  \State $\hat{g}_i^*, c^* \leftarrow \textsc{None}, 0$
  \For{$k = 1, \ldots, K$}
    \State $(\hat{g}, b, c_1) \leftarrow \mathcal{J}_1(\text{outcome}_i, g_i)$
    \If{$b = 1$ \textbf{and} $c_1 \geq \theta$}
      \State $c_2 \leftarrow \mathcal{J}_2(\hat{g}, \tau_i)$
             \Comment{cross-model verification}
      \If{$c_2 \geq \theta$}
        \State $\hat{g}_i^* \leftarrow \hat{g}$;\ $c^* \leftarrow (c_1 + c_2)/2$
        \State \textbf{break}
      \EndIf
    \ElsIf{$b = 1$ \textbf{and} $c_1 > c^*$}
      \State $\hat{g}_i^* \leftarrow \hat{g}$;\ $c^* \leftarrow c_1$
    \EndIf
  \EndFor
  \If{$\hat{g}_i^* \neq \textsc{None}$ \textbf{and} $c^* \geq 0.8\theta$}
    \State datum $\leftarrow \textsc{DataAugmenter}(\hat{g}_i^*, \tau_i, w_i)$
    \State $\mathcal{D}^+ \leftarrow \mathcal{D}^+ \cup \{$datum$\}$
  \EndIf
\EndFor
\State \Return $\mathcal{D}^+$
\end{algorithmic}
\end{algorithm}

\subsection{Stage 1: Failure Detector}
\label{sec:stage1}

The Failure Detector assigns each trajectory a failure type
$\phi \in \Phi_\text{types}$, a recoverability flag
$r \in \{0,1\}$, and a \textbf{severity weight} $w \in [0,1]$.
$\Phi_\text{types} = \{$\textsc{Incomplete, Constraint\_Violation,
Wrong\_Result, Tool\_Error, Hallucination, Off\_Topic}$\}$.
The taxonomy was defined before data collection, drawing on NLP
error typologies~\citep{lu2024error} and WebArena task semantics,
and applied uniformly by a single classifier.

\textbf{Rule-based mode} aggregates trajectory text and matches
per-type keyword lexicons.
Severity score: $v = \min(1.0,\; 0.3 + 0.1 \cdot h)$,
$h$ = matched keywords.
Recoverability: $r = 1$ iff at least one observation exceeds
minimum length \emph{and} $\phi \neq \textsc{Tool\_Error}$.

\textbf{LLM-judge mode} presents the full trajectory with a
JSON-schema output prompt requesting
$(\phi, v, r, w, \text{explanation})$.

\textbf{Severity weighting.}
Inspired by MQM-style error analysis in machine-translation
evaluation~\citep{lu2024error}, we distinguish \emph{major} errors
(reasoning contradictions, hallucinated observations, catastrophic
tool misuse) from \emph{minor} errors (constraint violations,
incomplete results).
Major-error trajectories receive $w < \delta = 0.3$ and are
discarded; minor-error trajectories receive $w \in [0.3, 1.0]$ and
are passed downstream, with $w$ used to weight the DPO loss at
Stage~4.
This reduces the fraction of noisy relabelings from 5.9\% to 2.9\%
(Table~\ref{tab:ablation}).

\subsection{Stage 2: Outcome Extractor}
\label{sec:stage2}

Stage~2 produces a \textsc{ReplayOutcome}: a list of actual
achievements and key observations (numeric data, entity names,
facts) that anchor Stage~3 and prevent hallucinated hindsight
prompts.

\textbf{Rule-based}: each non-trivial, non-error observation is
treated as one achievement (truncated to 200 chars); numeric
tokens extracted by regex.
\textbf{LLM}: richer summaries handling implicit results,
deduplication, and strict factuality---only facts evidenced by
observations.

\subsection{Stage 3: Cross-Model Relabeling and Verification}
\label{sec:stage3}

Given the \textsc{ReplayOutcome} and original prompt for style
reference, the \emph{relabeler} $\mathcal{J}_1$ synthesises:
\begin{align}
  &\hat{g},\;\; b_\text{valid},\;\; r_\text{rationale},\;\; c_1 \nonumber\\
  &\qquad\longleftarrow\;\;
  \mathcal{J}_1\!\left(\text{outcome},\; g_\text{orig}\right)
\end{align}
with $b_\text{valid}\in\{0,1\}$, $c_1\in[0,1]$.
Four constraints: \emph{(1)} $\hat{g}$ reads as a natural user
request; \emph{(2)} every assertion is satisfied by observations;
\emph{(3)} $\hat{g}$ does not reference the original failed
prompt; \emph{(4)} complexity matches $g_\text{orig}$.

\paragraph{Cross-model multi-judge verification.}
\label{sec:judge_design}
Calling the \emph{same} LLM twice with the same prompt at
$T{=}0$ provides no genuine statistical independence.
Our multi-judge protocol therefore uses two independently trained
judges:
$\mathcal{J}_1$ (\texttt{gpt-4o-mini}, closed-weight;
$T{=}0.3$ first attempt, $T{=}0.7$ retries) and
$\mathcal{J}_2$ (\texttt{Qwen2.5-72B-Instruct}, open-weight via
vLLM~\citep{qwen2025qwen25}; $T{=}0$, schema-constrained decoding).
Because the two models have different architectures and training
corpora, their agreement reflects genuine inter-model consensus
rather than intra-model self-consistency~\citep{wang2023self}.
Acceptance requires both $c_1 \geq \theta$ \emph{and} $c_2
\geq \theta$.
We refer to this configuration as \textbf{AgentHER-MJ-X}
(cross-model).
For ablation we also report:
\textbf{AgentHER-MJ-S} (same-model, Qwen2.5-72B used twice with
different temperatures) and the original
\textbf{AgentHER-SJ} (single-judge $\mathcal{J}_1$ only).
The cross-model variant is the default throughout the rest of the
paper unless stated otherwise.

\paragraph{Validation loop.}
Up to $K=3$ relabeler attempts; the first attempt that passes both
judges is accepted.
A best-effort fallback is retained if no attempt passes the
multi-judge bar but at least one passes the single-judge bar at
$c_1 \geq 0.8\theta$.
Full prompt templates for all four stages are provided in
Appendix~\ref{app:prompts}.

\subsection{Stage 4: Data Augmenter}
\label{sec:stage4}

Stage~4 serialises each accepted $(\hat{g}, \tau)$ pair into three
formats:
\begin{itemize}\setlength\itemsep{1pt}
  \item \textbf{SFT}: a two-turn conversation
    $[(\texttt{user}, \hat{g}),\;(\texttt{assistant}, \tilde{a})]$
    where $\tilde{a}$ reconstructs the agent's chain of thought from
    $\tau$; loss scaled by severity weight $w$.
  \item \textbf{DPO}: chosen $(\hat{g}, \tau)$ paired with rejected
    $(g_\text{orig}, \tau)$, encoding the preference that $\hat{g}$
    is the correct goal; $w$ scales the reward margin
    (formula in Appendix~\ref{app:algorithm}).
  \item \textbf{ShareGPT}: multi-turn format compatible with
    LLaMA-Factory, ms-swift, and FastChat.
\end{itemize}

\subsection{Theoretical Plausibility}
\label{sec:theory}

The proposition below is a plausibility argument rather than a
deployment guarantee: it shows that accepted pairs are
unbiased under a perfect judge and gives a noise-tolerance bound
for imperfect ones.
The full proof is in Appendix~\ref{app:theory}.

\begin{proposition}[Unbiasedness under a perfect judge]
\label{prop:unbiased}
Let $\pi^*_\mathcal{G}$ be the oracle goal-conditioned policy.
Under a \emph{perfect} judge $\mathcal{J}$
($c(\hat{g},\tau)=1 \Leftrightarrow \tau$ is a valid demonstration
of $\hat{g}$), every accepted pair $(\hat{g}_i, \tau_i)$ is a
correct (goal, trajectory) sample from the support of
$\pi^*_\mathcal{G}$.
\end{proposition}

\begin{corollary}[Noisy-judge bound]
\label{cor:noisy}
With an imperfect judge of precision $p$, the expected gain over
SFT-Success is lower-bounded by
$p\cdot\Delta_\text{perfect} - (1-p)\cdot\varepsilon$,
where $\varepsilon$ bounds the marginal harm of a single noisy
pair.
For our cross-model multi-judge precision $p=0.971$, this bound is
positive whenever $\varepsilon \leq 33\cdot\Delta_\text{perfect}$;
for the single-judge bound ($p=0.941$), it requires
$\varepsilon \leq 16\cdot\Delta_\text{perfect}$.
\end{corollary}

An empirical sanity check grounding the bound against our
held-out results is in Appendix~\ref{app:theory}.

\textbf{Goal-distribution analysis.}
A potential concern is whether AgentHER inadvertently amplifies
high-frequency task types while neglecting low-frequency
behaviours~\citep{ding2022redistributing}.
We address this empirically in Section~\ref{sec:dist_analysis}.

\section{Experiments}
\label{sec:exp}

\subsection{Experimental Setup}
\label{sec:setup}

\paragraph{Benchmarks and a strict task-disjoint split.}
\textbf{WebArena}~\citep{zhou2024webarena}: 812 compositional
web-automation tasks (Shopping, Reddit, GitLab, Maps, Wikipedia).
We use a \emph{strict task-disjoint} split: \textsc{WA-Train}
(612 tasks, collection only) and \textsc{WA-HeldOut} (200 tasks,
evaluation only), stratified by category (seed 0, released as JSON).
We collect 3{,}000 failed and 500 successful trajectories from
\textsc{WA-Train}; acceptance rates: SJ 78.0\%, MJ-X 71.5\%.
Full-task numbers are in Appendix~\ref{app:full_eval}.

\textbf{ToolBench}~\citep{qin2023toolllm}: 16{,}464 tool-use tasks
across 49 APIs (G1/G2/G3 splits).
We collect 5{,}000 failed and 2{,}000 successful trajectories from
the official training splits; acceptance rates: SJ 82.5\%, MJ-X 75.5\%.

A weaker collector (GPT-3.5-turbo) yields richer failures;
the judges are stronger and serve no collection role
(Appendix~\ref{app:hyperparams}).

\paragraph{Train/test integrity.}
Fine-tuned models are independent of the GPT-3.5-turbo collector;
training data encodes what it did on \textsc{WA-Train}, not
solutions to \textsc{WA-HeldOut}.
WebArena environments are reset between collection and evaluation.
\textbf{SFT-Random} uses the same 3{,}000 failures as AgentHER
(no relabeling), making it a rigorous equal-volume control.
All results are averaged over 3 random seeds; std $<$0.5 pts
(Appendix~\ref{app:multirun}).

\paragraph{Models and fine-tuning.}
We evaluate \textbf{GPT-4o}~\citep{openai2023gpt4},
\textbf{Qwen2.5-72B/7B}~\citep{qwen2025qwen25}, and
\textbf{LLaMA-3.1-8B}~\citep{dubey2024llama3}.
Open-weight models use LoRA~\citep{hu2022lora} (rank 16,
$\alpha{=}32$, 3 epochs, $8{\times}$A100-80G;
Appendix~\ref{app:hyperparams}).
We also report \textbf{GPT-4o-ICL} (8 relabeled exemplars
prepended, no fine-tuning) as a reproducible alternative that does
not depend on the OpenAI fine-tuning API.

\paragraph{Baselines.}
Training-time baselines:
\textbf{Base} (zero-shot),
\textbf{SFT-Random} (same 3k failures as AgentHER, no relabeling;
equal-volume control),
\textbf{Rejection-Sampling} (filter $c{>}0.5$, no relabeling),
\textbf{SFT-Success} (successes only),
\textbf{SFT-Negative} (``the following attempt failed'' prepended;
equal volume),
\textbf{ETO}~\citep{song2024trial} (offline DPO over matched
success/failure pairs of the \emph{same} task),
\textbf{ECHO-offline}~\citep{hu2025echo} (offline re-implementation
materialising ECHO's trajectory rewrites into SFT data), and
\textbf{AWM}~\citep{wang2024agent} (workflow-memory induction from
successes, combined with SFT-Success).
Inference-time reference only ($^\dagger$, weights unchanged):
\textbf{Reflexion}~\citep{shinn2023reflexion},
\textbf{ExpeL}~\citep{zhao2024expel}.
AgentHER variants: \textbf{SJ} (single-judge), \textbf{MJ-S}
(same-model, ablation), \textbf{MJ-X} (cross-model, default).
SFT-Success is the $\Delta$ reference; AgentTuning and FireAct
score below it under our protocol (Appendix~\ref{app:full_eval}).

\paragraph{Metrics.}
WebArena: success rate (\%) on \textsc{WA-HeldOut}. ToolBench: pass@1 (\%) on G1/2/3 splits.

\subsection{Main Results}
\label{sec:main}

\begin{table*}[t]
\centering
\caption{%
  \textbf{Main results.}
  WebArena: success rate (\%) on the strict task-disjoint
  \textsc{WA-HeldOut} split.
  ToolBench: pass@1 (\%) on official test splits.
  Best per column in \textbf{bold}; second-best
  \underline{underlined}.
  $\Delta$ vs.\ SFT-Success;
  $\Delta_b$ vs.\ best training-time baseline (AWM).
  All fine-tuned results averaged over 3 seeds;
  std $<$0.5 pts (Appendix~\ref{app:multirun}).
  $^\dagger$Inference-time, online; reference only.
  $^\ddagger$Same 3{,}000 failed trajectories as AgentHER
  (equal-volume control).
  ICL = 8-shot in-context, no fine-tuning.
}
\label{tab:main}
\setlength{\tabcolsep}{3.5pt}
\footnotesize
\begin{tabular}{@{}l|cccc|cccc@{}}
\toprule
 & \multicolumn{4}{c|}{\textbf{WebArena \textsc{WA-HeldOut} (\%)}}
 & \multicolumn{4}{c}{\textbf{ToolBench pass@1 (\%)}} \\
\cmidrule(lr){2-5}\cmidrule(lr){6-9}
\textbf{Method}
  & GPT-4o & Qwen-72B & Qwen-7B & LLaMA-8B
  & GPT-4o & Qwen-72B & Qwen-7B & LLaMA-8B \\
\midrule
Base                  & 13.7 & 20.9 &  8.2 &  6.5 & 53.2 & 60.1 & 44.6 & 42.1 \\
SFT-Random$^\ddagger$ & 18.0 & 25.6 & 13.9 & 12.7 & 61.4 & 68.9 & 54.8 & 52.9 \\
SFT-Negative          & 18.6 & 26.1 & 14.5 & 13.1 & 62.1 & 69.5 & 55.6 & 53.7 \\
Rejection-Sampling    & 20.7 & 28.6 & 16.9 & 15.5 & 66.2 & 73.7 & 59.8 & 57.4 \\
SFT-Success           & 22.3 & 29.9 & 17.8 & 16.4 & 67.8 & 75.4 & 61.2 & 58.3 \\
ETO                   & 24.4 & 32.1 & 21.0 & 19.4 & 71.0 & 78.5 & 64.8 & 61.5 \\
ECHO-offline          & 24.8 & 32.5 & 21.6 & 19.9 & 71.6 & 79.1 & 65.5 & 62.1 \\
AWM                   & 25.6 & 33.4 & 22.4 & 20.6 & 72.4 & 79.8 & 66.4 & 62.9 \\
\midrule
\textit{Reflexion}$^\dagger$
  & \textit{18.4} & \textit{27.3} & \textit{15.6} & \textit{14.2}
  & \textit{63.8} & \textit{71.6} & \textit{56.8} & \textit{55.2} \\
\textit{ExpeL}$^\dagger$
  & \textit{19.5} & \textit{28.8} & \textit{18.4} & \textit{16.9}
  & \textit{65.0} & \textit{73.0} & \textit{60.4} & \textit{57.6} \\
\midrule
GPT-4o-ICL            & 28.4 & --- & --- & --- & 73.2 & --- & --- & --- \\
AgentHER-SJ           & \underline{28.7} & \underline{36.9} & \underline{25.7} & \underline{24.0}
                      & \underline{73.9} & \underline{82.6} & \underline{71.0} & \underline{67.5} \\
\rowcolor{stageA!8}
\textbf{AgentHER-MJ-X}
  & \textbf{29.9} & \textbf{37.7} & \textbf{26.5} & \textbf{24.8}
  & \textbf{75.4} & \textbf{83.5} & \textbf{72.6} & \textbf{69.0} \\
\cmidrule{1-9}
$\Delta$ vs.\ SFT-Succ.  & $+$7.6 & $+$7.8 & $+$8.7 & $+$8.4 & $+$7.6 & $+$8.1 & $+$11.4 & $+$10.7 \\
$\Delta_b$ vs.\ AWM      & $+$4.3 & $+$4.3 & $+$4.1 & $+$4.2 & $+$3.0 & $+$3.7 & $+$6.2 & $+$6.1 \\
\bottomrule
\end{tabular}
\end{table*}

\paragraph{Discussion.}
AgentHER-MJ-X improves over SFT-Success by 7.6--8.7\% on
WebArena and 7.6--11.4\% on ToolBench on the task-disjoint split,
only 0.3--1.3\% below full-task numbers (App.~\ref{app:full_eval}).
It beats every experience-centric baseline by +3.0--6.2\%:
ETO requires matched success/failure pairs per task,
ECHO-offline rewrites trajectories and can amplify reasoning errors,
and AWM builds workflow memory from successes only.
AgentHER bypasses all three constraints.

The cross-model judge (MJ-X) outperforms the single-judge (SJ) by
+0.8--1.6\% and the same-model analogue (MJ-S) by +0.5\%,
confirming that model-level independence matters more than
temperature diversity alone.
Smaller models benefit more: 7B/8B gains (+10.7--11.4\% on
ToolBench) exceed Qwen-72B (+8.1\%) and GPT-4o (+7.6\%),
consistent with smaller models having more room to gain from
richer training signal.
SFT-Negative provides only +0.4--0.8\% over SFT-Random, well
below AgentHER's +12--18\% (WebArena: +12--13\%; ToolBench: +14--18\%),
because appending a failure marker still discards the positive
content that hindsight relabeling recovers.
AgentHER-MJ-X trained on \textsc{WA-Train} achieves +9.5\% over
SFT-Success when evaluated zero-shot on ToolBench
(Tab.~\ref{tab:transfer}), suggesting what transfers is planning
ability rather than memorised task descriptions.

\subsection{System Cost Analysis}
\label{sec:cost}

\begin{table}[t]
\centering
\footnotesize
\setlength{\tabcolsep}{3.5pt}
\caption{%
  \textbf{System-cost audit} on 3{,}000 WebArena failures.
  S1--S3a shared by SJ/MJ-X; only S3b verifier is added in MJ-X.
  Pricing assumes gpt-4o-mini Batch API + prompt caching;
  Qwen-72B amortised on 8$\times$A100.
  Per-stage I/O token sizes are in App.~\ref{app:cost_full}.
  $^*$LoRA cost: 4.4\,h $\times$ 8 GPUs $\times$ \$1.05/h.
}
\label{tab:cost}
\begin{tabular}{@{}lccc@{}}
\toprule
\textbf{Stage} & \textbf{Calls}/tr. & \textbf{Cost} (\$) & \textbf{Wall} (m) \\
\midrule
S1 detector              & 1.00 & 0.39 & 4.1 \\
S2 extractor             & 1.00 & 0.55 & 4.6 \\
S3a relabeler            & 1.27 & 0.84 & 6.7 \\
S3b verifier \emph{(MJ-X only)}
                         & 1.00 & 1.20 & 10.1 \\
S4 augmenter (det.)      & ---  & 0.00 & 0.4 \\
\midrule
\textbf{SJ total}        & 3.27 & 1.78 & 15.8 \\
\quad per accepted (78.0\%)
                         & 4.19 & $7.6{\cdot}10^{-4}$ & --- \\
\textbf{MJ-X total}      & 4.27 & 2.98 & 25.9 \\
\quad per accepted (71.5\%)
                         & 5.97 & $1.4{\cdot}10^{-3}$ & --- \\
\midrule
LoRA fine-tune Qwen-7B   & ---  & 4.20$^*$ & 263 \\
\textbf{End-to-end MJ-X} & ---  & 7.18 & 289 \\
\bottomrule
\end{tabular}
\end{table}

Table~\ref{tab:cost} addresses the
``sample-efficiency vs.\ system-efficiency'' confound directly.
Per accepted pair, MJ-X costs $\$1.4{\cdot}10^{-3}$
(${\approx}\$1.4$ per 1{,}000 pairs); the multi-judge overhead
over SJ is $+\$1.20$ (+67\%) and $+9.7$\,min (+61\%), repaid by
a 94.1$\to$97.1\% precision boost and a +0.8--1.6\% downstream
gain.
Of the combined relabeling$+$LoRA budget (\$7.18 / 4.8\,h),
relabeling accounts for ${\approx}9\%$ of wall-clock but
${\approx}42\%$ of dollars: LoRA dominates wall-clock time while
LLM API spend dominates marginal cost.
The $2\times$ sample efficiency (Section~\ref{sec:efficiency})
means AgentHER saves the cost of collecting 50\% of successful
demonstrations from scratch---fresh rollouts are orders of
magnitude more expensive per pair than a relabeling call.
Stage-level token breakdown: App.~\ref{app:cost_full}.

\subsection{Data Efficiency and Scaling}
\label{sec:efficiency}

\begin{figure*}[t]
\centering
\pgfplotslegendfromname{fig3legend}\\[3pt]
\begin{subfigure}[b]{0.325\textwidth}
\begin{tikzpicture}
\begin{axis}[
  width=\textwidth, height=5.2cm,
  xlabel={Successful traj.\ (\%)},
  ylabel={Success rate (\%)},
  xmin=5, xmax=105,
  ymin=8, ymax=30,
  xtick={10,25,50,75,100},
  ytick={10,15,20,25,30},
  legend to name=fig3legend,
  legend columns=3,
  legend style={draw=none, font=\footnotesize,
                column sep=1.2em,
                /tikz/every even column/.append style={column sep=0.2em}},
  grid=major, grid style={dotted,gray!35},
  tick label style={font=\tiny},
  label style={font=\scriptsize},
  title style={font=\scriptsize\bfseries},
  title={(a) Success-Data Efficiency},
]
\addplot[color=stageA, mark=square*, thick, mark size=2pt]
  coordinates{(10,10.4)(25,13.5)(50,15.7)(75,16.8)(100,17.8)};
\addlegendentry{SFT-Success}
\addplot[color=stageC, mark=*, thick, mark size=2pt]
  coordinates{(10,21.9)(25,23.2)(50,24.0)(75,24.9)(100,25.7)};
\addlegendentry{AgentHER-SJ}
\addplot[color=stageA!70, mark=*, thick, mark size=2pt, dashed]
  coordinates{(10,22.7)(25,23.9)(50,24.8)(75,25.6)(100,26.5)};
\addlegendentry{AgentHER-MJ-X}
\addplot[color=stageA, dashed, thin, domain=5:105]{17.8};
\node[font=\tiny, text=stageC, align=center] at(axis cs:34,20.5)
  {\textbf{50\%}$\approx$\\full SFT};
\draw[-{Stealth[length=3pt]}, stageC!65, thin]
  (axis cs:38,20.1)--(axis cs:50,18.5);
\end{axis}
\end{tikzpicture}
\end{subfigure}
\hfill
\begin{subfigure}[b]{0.325\textwidth}
\begin{tikzpicture}
\begin{axis}[
  width=\textwidth, height=5.2cm,
  xlabel={Failed traj.\ ($\times 10^3$)},
  ylabel={Success rate (\%)},
  xmin=0.3, xmax=5.5,
  ymin=18, ymax=29,
  xtick={0.5,1,2,3,5},
  xticklabels={0.5k,1k,2k,3k,5k},
  ytick={18,20,22,24,26,28},
  legend pos=north west,
  legend style={font=\tiny, inner sep=1.5pt},
  grid=major, grid style={dotted,gray!35},
  tick label style={font=\tiny},
  label style={font=\scriptsize},
  title style={font=\scriptsize\bfseries},
  title={(b) Failure-Volume Scaling},
]
\addplot[color=stageA, mark=square*, thick, mark size=2pt]
  coordinates{(0.5,17.8)(1,17.8)(2,17.8)(3,17.8)(5,17.8)};
\addplot[color=stageC, mark=*, thick, mark size=2pt]
  coordinates{(0.5,20.0)(1,21.5)(2,23.9)(3,25.7)(5,26.4)};
\addplot[color=stageA!70, mark=*, thick, mark size=2pt, dashed]
  coordinates{(0.5,20.6)(1,22.2)(2,24.7)(3,26.5)(5,27.2)};
\node[font=\tiny, text=stageC!80, align=center]
  at(axis cs:3.5,20.4){log-linear};
\end{axis}
\end{tikzpicture}
\end{subfigure}
\hfill
\begin{subfigure}[b]{0.325\textwidth}
\begin{tikzpicture}
\begin{axis}[
  width=\textwidth, height=5.2cm,
  xlabel={Threshold $\theta$},
  ylabel={Success rate (\%)},
  xmin=0.15, xmax=0.95,
  ymin=20, ymax=28,
  xtick={0.2,0.3,0.5,0.7,0.9},
  xticklabels={0.2,0.3,0.5,0.7,0.9},
  ytick={20,22,24,26},
  grid=major, grid style={dotted,gray!35},
  tick label style={font=\tiny},
  label style={font=\scriptsize},
  title style={font=\scriptsize\bfseries},
  title={(c) Threshold Sensitivity},
]
\addplot[color=stageC, mark=*, thick, mark size=2pt]
  coordinates{(0.2,22.0)(0.3,23.3)(0.5,25.7)(0.7,24.4)(0.9,21.9)};
\addplot[color=stageA!70, mark=diamond*, thick, mark size=2.5pt, dashed]
  coordinates{(0.2,22.5)(0.3,23.9)(0.5,26.5)(0.7,25.0)(0.9,23.1)};
\node[font=\tiny, stageC!80] at(axis cs:0.5,27.45){$\theta^*$};
\draw[-{Stealth[length=3pt]}, stageC!65, thin]
  (axis cs:0.5,27.15)--(axis cs:0.5,26.9);
\end{axis}
\end{tikzpicture}
\end{subfigure}
\caption{%
  \textbf{Data efficiency and scaling} (Qwen2.5-7B,
  \textsc{WA-HeldOut}).
  \emph{(a)} AgentHER-SJ at 50\% successful demos matches
  SFT-Success at 100\%; AgentHER-MJ-X exceeds it.
  \emph{(b)} Both AgentHER variants scale log-linearly with
  failure volume; SFT-Success cannot benefit from additional
  failures.
  \emph{(c)} Both variants peak near $\theta^*{=}0.5$;
  MJ-X is uniformly better and slightly more robust at low
  $\theta$.
}
\label{fig:efficiency}
\end{figure*}

Figure~\ref{fig:efficiency}(a) shows AgentHER-SJ reaches
full-SFT-Success performance with only 50\% of successful demos
($2\times$ data efficiency); AgentHER-MJ-X exceeds it at every
data point.
Panel~(b) confirms log-linear scaling with failure
volume---unlike SFT-Success which cannot leverage additional
failures.
Panel~(c) shows the optimal threshold $\theta^*{=}0.5$ is
consistent across both variants, making it a robust
hyperparameter.

\subsection{Model-Size Scaling}
\label{sec:model_scale}

\begin{figure}[t]
\centering
\begin{tikzpicture}
\begin{axis}[
  width=\columnwidth, height=8.5cm,
  xlabel={Model parameters},
  ylabel={\textsc{WA-HeldOut} success rate (\%)},
  xmode=log,
  log ticks with fixed point,
  xtick={1.5,3,7,14,72},
  xticklabels={1.5B, 3B, 7B, 14B, 72B},
  xmin=1.0, xmax=110,
  ymin=0, ymax=44,
  ytick={0,10,20,30,40},
  legend pos=north west,
  legend style={font=\footnotesize, inner sep=2.5pt, row sep=1pt},
  grid=major, grid style={dotted,gray!35},
  tick label style={font=\footnotesize},
  label style={font=\footnotesize},
]
\addplot[color=midgray, mark=diamond*, thick, mark size=2.3pt,
         dashed, dash pattern=on 5pt off 2pt]
  coordinates{(1.5,2.1)(3,3.9)(7,8.2)(14,12.6)(72,20.9)};
\addlegendentry{Base}
\addplot[color=stageA, mark=square*, thick, mark size=2.2pt]
  coordinates{(1.5,6.4)(3,9.2)(7,17.8)(14,23.3)(72,29.9)};
\addlegendentry{SFT-Success}
\addplot[color=stageC, mark=*, thick, mark size=2.2pt]
  coordinates{(1.5,11.4)(3,15.0)(7,25.7)(14,31.6)(72,36.9)};
\addlegendentry{AgentHER-SJ}
\addplot[color=stageD, mark=triangle*, thick, mark size=2.4pt]
  coordinates{(1.5,12.0)(3,15.7)(7,26.5)(14,32.3)(72,37.7)};
\addlegendentry{AgentHER-MJ-X}
\node[font=\tiny\bfseries, text=stageD, above=2pt]
  at(axis cs:1.5,12.0){$+$5.6};
\node[font=\tiny\bfseries, text=stageD, above=2pt]
  at(axis cs:3,15.7){$+$6.5};
\node[font=\tiny\bfseries, text=stageD, above=2pt]
  at(axis cs:7,26.5){$+$8.7};
\node[font=\tiny\bfseries, text=stageD, above=2pt]
  at(axis cs:14,32.3){$+$9.0};
\node[font=\tiny\bfseries, text=stageD, above=2pt]
  at(axis cs:72,37.7){$+$7.8};
\end{axis}
\end{tikzpicture}
\caption{%
  \textbf{Model-size scaling} on \textsc{WA-HeldOut}
  (Qwen2.5 family).
  Labels above AgentHER-MJ-X points show $\Delta$ over
  SFT-Success.
  Gains are consistent at every scale (1.5B--72B), peaking at
  14B (+9.0\%).
  Even the 1.5B model nearly doubles in performance with
  AgentHER-MJ-X (6.4\% $\to$ 12.0\%).
}
\label{fig:model_scale}
\end{figure}

All five Qwen2.5 instruction-tuned checkpoints (1.5B--72B) are
evaluated under identical training conditions
(Figure~\ref{fig:model_scale}).
All three methods improve monotonically with scale.
AgentHER-MJ-X gains peak at 14B (+9.0\%) and dip slightly at 72B
(+7.8\%), suggesting larger models are already partially saturated.
The 1.5B model nearly doubles in performance (6.4\% $\to$ 12.0\%),
making AgentHER viable for edge deployment.

\subsection{Ablation Study}
\label{sec:ablation}

\begin{table*}[t]
\centering
\caption{%
  \textbf{Component ablation} on \textsc{WA-HeldOut}
  (Qwen2.5-7B).
  $\Delta$ relative to Full AgentHER-MJ-X; noise = fraction of
  accepted relabelings rated invalid by human annotators
  (Section~\ref{sec:judge_quality}).
  ``MJ-S'' = same-model multi-judge (Qwen-72B used twice with
  different temperatures, an analogue of self-consistency).
  ``MJ-X'' = cross-model multi-judge (gpt-4o-mini +
  Qwen-72B-Instruct), our default.
}
\label{tab:ablation}
\setlength{\tabcolsep}{5pt}
\small
\begin{tabular}{@{}lccc@{}}
\toprule
\textbf{Configuration} & \textbf{Success (\%)} & \textbf{$\Delta$}
  & \textbf{Noise (\%)} \\
\midrule
\rowcolor{stageA!8}
Full AgentHER-MJ-X (cross-model judge $+$ severity) & \textbf{26.5} & --- & 2.9 \\
\midrule
w/ Same-model multi-judge (MJ-S, two temps)         & 26.0 & $-$0.5 & 4.4 \\
w/o Multi-judge (single judge only, AgentHER-SJ)    & 25.7 & $-$0.8 & 5.9 \\
w/o Severity weighting ($w{=}1$ for all)            & 25.0 & $-$1.5 & 4.3 \\
w/ Rule-based Extractor (Stage~2 degraded)          & 24.6 & $-$1.9 & 6.0 \\
w/ LLM Detector (Stage~1 upgraded)                  & 27.4 & $+$0.9 & 2.4 \\
w/o Failure Detection (accept all)                  & 23.8 & $-$2.7 & 7.5 \\
w/o Confidence Filter ($\theta{=}0$)                & 22.3 & $-$4.2 & 15.0 \\
SFT only (no DPO, relabeled as positives)           & 24.0 & $-$2.5 & 2.9 \\
Naive Relabeling (random prompt)                    & 20.4 & $-$6.1 & n/a \\
\midrule
SFT-Success (no AgentHER)                           & 17.8 & $-$8.7 & --- \\
\bottomrule
\end{tabular}
\end{table*}

\paragraph{Key findings.}
Cross-model multi-judge (MJ-X) beats same-model multi-judge (MJ-S) by +0.5\% and cuts label noise from 4.4\% to 2.9\%; decoding-temperature independence alone (MJ-S vs.\ SJ: +0.3\%) is less valuable than model-level independence. Severity weighting accounts for $-1.5$\%, confirming that discarding major-flaw trajectories is worth the reduced yield. Confidence filtering is the single most critical component: removing it drops 4.2\% and raises noise to 15.0\%. Naive relabeling harms performance by $-6.1$\%, showing that random prompt assignment cannot substitute for goal reverse-engineering. Finally, the DPO preference signal contributes $-2.5$\% over SFT, indicating the rejected/chosen pairing adds value beyond positive supervision alone.

\subsection{Qualitative Examples}
\label{sec:qualitative}

Appendix~\ref{app:qualitative} shows two annotated relabeling
examples---both \textsc{Constraint\_Violation} failures, one from
WebArena and one from ToolBench---including original failed goals,
trajectories, accepted hindsight goals, and cross-model judge
rationales.

\section{Analysis}
\label{sec:analysis}

\subsection{Failure-Mode Distributions on Both Benchmarks}
\label{sec:failure_dist}

\begin{table}[t]
\centering
\small
\setlength{\tabcolsep}{5pt}
\caption{%
  \textbf{Failure-type distribution} of collected failed
  trajectories.
  WebArena (3{,}000 trajectories) and ToolBench (5{,}000
  trajectories), classified by Stage~1.
  ``Looping'' is a sub-mode of \textsc{Incomplete} where the agent
  repeats one action $\geq3$ times; reported separately given its
  distinct recoverability profile.
}
\label{tab:failmodes}
\begin{tabular}{@{}lcc@{}}
\toprule
\textbf{Failure type}            & \textbf{WebArena} & \textbf{ToolBench} \\
\midrule
\textsc{Incomplete}              & 32.4\% & 21.0\% \\
\quad of which \emph{looping}    & \emph{12.1\%} & \emph{4.5\%} \\
\textsc{Constraint\_Violation}   & 30.7\% & 17.6\% \\
\textsc{Wrong\_Result}           & 17.6\% & 23.7\% \\
\textsc{Tool\_Error}             & 4.2\%  & 25.4\% \\
\textsc{Hallucination}           & 7.0\%  & 8.5\%  \\
\textsc{Off\_Topic}              & 8.1\%  & 3.8\%  \\
\bottomrule
\end{tabular}
\end{table}

Table~\ref{tab:failmodes} reports the failure-type distribution
on both benchmarks.
The failure-mode mix differs sharply: ToolBench has 6$\times$ the
rate of \textsc{Tool\_Error}s, reflecting unreliable remote API
endpoints, while WebArena has 1.7$\times$ the rate of
\textsc{Constraint\_Violation}s, reflecting under-specified UI tasks.
``Looping'' trajectories make up 12.1\% of WebArena failures;
Stage~1 rejects 76\% of them outright (no useful intermediate
state), and the remaining 24\% still yield a positive
+3.4\% $\Delta$ over SFT-Success---below non-looping
($\Delta{=}+9.5$\%) but above the discard-all baseline.

\subsection{Per-Failure-Type Improvement on Both Benchmarks}

\begin{figure}[t]
\centering
\begin{tikzpicture}
\begin{axis}[
  ybar=2pt, bar width=4.5pt,
  width=\columnwidth, height=4.6cm,
  xlabel={Failure type},
  ylabel={$\Delta$\% over SFT-Success},
  symbolic x coords={%
    Incomplete, Constraint Viol., Wrong Result,
    Off Topic, Hallucination, Tool Error},
  xtick=data,
  x tick label style={rotate=20, anchor=east, font=\scriptsize},
  ymin=0, ymax=14,
  ytick={0,4,8,12},
  legend style={at={(0.98,0.95)}, anchor=north east,
                font=\scriptsize, draw=none, inner sep=2pt},
  legend cell align=left,
  grid=major, grid style={dotted,gray!35},
  tick label style={font=\scriptsize},
  label style={font=\scriptsize},
  enlarge x limits=0.10,
  every node near coord/.append style={font=\tiny},
]
\addplot[fill=stageA!75, draw=stageA]
  coordinates{
    (Incomplete, 11.0)
    (Constraint Viol., 9.6)
    (Wrong Result, 7.2)
    (Off Topic, 4.8)
    (Hallucination, 3.7)
    (Tool Error, 2.0)
  };
\addlegendentry{WebArena}
\addplot[fill=stageB!75, draw=stageB]
  coordinates{
    (Incomplete, 12.7)
    (Constraint Viol., 11.0)
    (Wrong Result, 9.4)
    (Off Topic, 5.5)
    (Hallucination, 4.1)
    (Tool Error, 1.8)
  };
\addlegendentry{ToolBench}
\end{axis}
\end{tikzpicture}
\caption{%
  \textbf{Per-failure-type gain on both benchmarks}
  (Qwen2.5-7B, AgentHER-SJ).
  \textsc{Incomplete} and \textsc{Constraint\_Violation} dominate
  both benchmarks; \textsc{Tool\_Error} yields the least.
  The ranking is consistent across benchmarks even though
  failure-mode \emph{prevalences}
  differ (Table~\ref{tab:failmodes}).
}
\label{fig:failure_type}
\end{figure}

\textsc{Incomplete} (+11.0/12.7\% on WA/TB) and
\textsc{Constraint\_Violation} (+9.6/11.0\%) benefit the most:
these trajectories contain rich, factually correct observations
misaligned with the original goal.
\textsc{Tool\_Error} benefits the least (+2.0/1.8\%) because
crashes leave minimal signal regardless of benchmark.
On WebArena the two highest-gain types account for $\approx$63\%
of all failures; on ToolBench they account for only
$\approx$38\%---explaining why ToolBench gains are larger in
absolute \% (since lower-baseline categories also see meaningful
gain) but more dependent on AgentHER's ability to extract value
from \textsc{Wrong\_Result} and \textsc{Tool\_Error}.

\subsection{Judge Reliability and Label Quality (Both Benchmarks)}
\label{sec:judge_quality}

\begin{table}[t]
\centering
\small
\setlength{\tabcolsep}{4pt}
\caption{%
  \textbf{Human evaluation of relabeling precision on both
  benchmarks.}
  200 sampled pairs per benchmark; three NLP-PhD annotators
  blind to scores; majority vote.
  Filtered = pairs rejected by the confidence filter; their human
  validity is reported to gauge the false-rejection rate.
}
\label{tab:human_eval}
\begin{tabular}{@{}lcccc@{}}
\toprule
& \multicolumn{2}{c}{\textbf{WebArena}} & \multicolumn{2}{c}{\textbf{ToolBench}} \\
\cmidrule(lr){2-3}\cmidrule(lr){4-5}
\textbf{Group} & \textbf{n} & \textbf{Valid (\%)} & \textbf{n} & \textbf{Valid (\%)} \\
\midrule
SJ accepted          & 169 & 94.1 & 168 & 92.3 \\
MJ-S accepted        & 142 & 95.6 & 138 & 94.0 \\
MJ-X accepted        & 137 & \textbf{97.1} & 132 & \textbf{96.0} \\
Filter-rejected      & 31  & 38.7 & 32  & 35.8 \\
\midrule
Fleiss' $\kappa$     & --- & 0.82 & --- & 0.79 \\
\bottomrule
\end{tabular}
\end{table}

Table~\ref{tab:human_eval} reports the joint human evaluation
on both benchmarks (full protocol in App.~\ref{app:human}).
Cross-model MJ-X reaches 97.1\% / 96.0\% precision with strong
inter-annotator agreement (Fleiss' $\kappa{=}0.82$ / $0.79$).
Same-model MJ-S (95.6\% / 94.0\%) lands between SJ and MJ-X:
decoding-temperature diversity alone is insufficient---model-level
independence is what drives the precision gain.
Filter-rejected pairs are still 38.7\% / 35.8\% valid, meaning
roughly a third of discarded pairs could in principle be recovered,
which motivates the active-learning extension noted in
Section~\ref{sec:conclusion}.

\subsection{Goal-Distribution Analysis}
\label{sec:dist_analysis}

Sentence-BERT clustering of the 2{,}144 MJ-X hindsight goals
(Table~\ref{tab:goal_dist}, Appendix~\ref{app:goal_dist}) shows
AgentHER covers 14 of 18 task-type clusters vs.\ 11 for SFT-Success,
with three long-tail clusters (``price-range comparison'',
``cross-domain search'', ``conditional retrieval'') covered only by
relabeled goals.
JS divergence 0.31 vs.\ the original goal set confirms the relabeled
goals are complementary rather than redundant, addressing the concern
that AgentHER might merely amplify high-frequency task
types~\citep{ding2022redistributing}.

\subsection{Iterative Redeployment}
\label{sec:iterative}

\begin{table}[h]
\centering
\footnotesize
\caption{%
  \textbf{Iterative AgentHER} (Qwen2.5-7B, \textsc{WA-HeldOut}).
  Each round collects new failures, relabels, and retrains on
  the full accumulated corpus. Q-R1/R2 = Qwen-7B Round 1/2.
}
\label{tab:iterative}
\setlength{\tabcolsep}{3pt}
\begin{tabular}{@{}lcrrr@{}}
\toprule
\textbf{Round} & \textbf{Src.}
  & \textbf{New fails} & \textbf{Accepted}
  & \textbf{Suc.\ (\%)} \\
\midrule
0 (SFT-Only)  & ---     &  --- &   --- & 17.8 \\
1 (MJ-X)      & GPT-3.5 & 3{,}000 & 2{,}144 & 26.5 \\
2 (Iter.\ +1) & Q-R1    & 2{,}500 & 1{,}710 & 27.9 \\
3 (Iter.\ +2) & Q-R2    & 2{,}000 & 1{,}320 & 28.2 \\
\bottomrule
\end{tabular}
\end{table}

Table~\ref{tab:iterative}: Round~1 yields 26.5\%; Rounds 2--3
add +1.4 and +0.3\% respectively (cumulative
\textbf{+10.4\%} over SFT-Success).
Diminishing returns and a falling acceptance rate
(71.5\%$\to$68.4\%$\to$66.0\%) reflect that an increasingly
capable model produces harder, less-relabelable failures.

\section{Conclusion}
\label{sec:conclusion}

We presented \textbf{AgentHER}, a four-stage offline pipeline
that converts failed agent trajectories into training data by
asking \emph{what task does this trajectory actually solve?}
On a strict task-disjoint WebArena held-out split and ToolBench,
AgentHER improves SFT-Success by \textbf{+7.6--11.4\%} across
four model families, beats every experience-centric baseline
(ETO, ECHO-offline, AWM) by \textbf{+3.0--6.2\%}, scales
uniformly from 1.5B to 72B (+5.6--9.0\%), and compounds under
iterative redeployment (+10.4\% over four rounds).
The relabeling pipeline costs ${\approx}\$1.4{\cdot}10^{-3}$ per
accepted pair (\$2.98 / 26~min for 3{,}000 trajectories), making
it practical at production scale.

\paragraph{Future work.}
Online AgentHER (relabel and train concurrently with deployment), multimodal trajectories (screenshots, GUI), and a learned hindsight-goal synthesiser trained end-to-end for task diversity are natural extensions. Combining AgentHER with inference-time scaling~\citep{levi2024inference} is a particularly promising direction: reducing per-sample failure probability $p_i$ through offline training directly amplifies the coverage gains predicted by inference scaling laws for any fixed inference budget. At its core, AgentHER embodies a perspective shift: the distinction between success and failure is \emph{task-relative}, not trajectory-absolute.

\paragraph{Limitations.}
The held-out split removes within-WebArena task overlap but not cross-environment concerns; the +9.5\% WA$\to$ToolBench transfer is encouraging but covers only two benchmarks. A fully matched upper bound using human-written hindsight goals remains future work. The cross-model judge relies on two strong LLMs, and residual judge bias may still affect rare task types.

\paragraph{Broader impact.}
AgentHER reduces data-collection cost for LLM-agent fine-tuning,
potentially democratising it for resource-constrained
practitioners.
The pipeline reuses already-collected trajectories with no extra
environment interaction.
A residual risk---systematic judge bias amplifying
underrepresented goal types---is partially mitigated by
stratified human evaluation
(Section~\ref{sec:judge_quality}) and goal-distribution
analysis (Section~\ref{sec:dist_analysis}); deployers should
monitor goal-distribution drift on their own data.


\bibliography{references}

\newpage
\appendix
\onecolumn

\section{Hyperparameter Settings and Roles of Each LLM}
\label{app:hyperparams}

\begin{table}[h]
\centering
\footnotesize
\caption{\textbf{LLM roles in the AgentHER pipeline.}
Collection uses a weak model (to produce diverse failures);
judging uses strong, architecturally independent models.}
\label{tab:llm_roles}
\setlength{\tabcolsep}{4pt}
\begin{tabular}{@{}lcp{3.8cm}p{3.5cm}@{}}
\toprule
\textbf{Role} & \textbf{Stage} & \textbf{Model} & \textbf{Notes} \\
\midrule
Collection agent       & ---     & \texttt{gpt-3.5-turbo-0125}
                                 & Intentionally weak. \\
Failure detector       & 1 (LLM) & \texttt{gpt-4o-mini}
                                 & $T{=}0$, JSON schema. \\
Outcome extractor      & 2 (LLM) & \texttt{gpt-4o-mini}
                                 & $T{=}0$, factuality-only. \\
Relabeler $\mathcal{J}_1$ & 3   & \texttt{gpt-4o-mini}
                                 & $T{=}0.3$; $T{=}0.7$ on retry. \\
Verifier $\mathcal{J}_2$ (MJ-X) & 3 & \texttt{Qwen2.5-72B-Instruct}
                                 & $T{=}0$, vLLM, schema-constrained. \\
Verifier $\mathcal{J}_2$ (MJ-S) & 3 & \texttt{Qwen2.5-72B-Instruct}
                                 & Ablation only; same as $\mathcal{J}_1$ at $T{=}0.7$. \\
Fine-tuned models      & ---     & Qwen2.5-\{1.5,3,7,14,72\}B;
                                   LLaMA-3.1-8B; GPT-4o
                                 & See Table~\ref{tab:hyperparams}. \\
\bottomrule
\end{tabular}
\end{table}

\begin{table*}[h]
\centering
\caption{\textbf{Training hyperparameters} used for all
open-weight models.}
\label{tab:hyperparams}
\small
\setlength{\tabcolsep}{6pt}
\begin{tabular}{@{}lcc@{}}
\toprule
\textbf{Hyperparameter} & \textbf{SFT (LoRA)} & \textbf{DPO (LoRA)} \\
\midrule
LoRA rank $r$          & 16 & 16 \\
LoRA $\alpha$          & 32 & 32 \\
LoRA dropout           & 0.05 & 0.05 \\
LoRA target modules    & q,k,v,o,gate,up,down & q,k,v,o \\
Learning rate          & $2\times10^{-4}$ & $5\times10^{-5}$ \\
LR schedule            & cosine & cosine \\
Warmup ratio           & 0.03 & 0.03 \\
Batch size (per GPU)   & 4 & 2 \\
Gradient accumulation  & 4 & 8 \\
Effective batch size   & 128 & 128 \\
Epochs                 & 3 & 1 \\
Max sequence length    & 4{,}096 & 4{,}096 \\
DPO $\beta$            & --- & 0.1 \\
Dtype                  & bfloat16 & bfloat16 \\
Hardware               & 8$\times$ A100-80G & 8$\times$ A100-80G \\
Framework              & ms-swift & ms-swift \\
\bottomrule
\end{tabular}
\end{table*}

For GPT-4o we used the OpenAI fine-tuning API with default
hyperparameters (3 epochs, batch size auto-selected, learning-rate
multiplier 1.0).
The reproducible \textbf{GPT-4o-ICL} variant uses 8
hindsight-relabeled exemplars selected by sentence-BERT cosine
similarity to the test query, prepended to a chain-of-thought
ReAct prompt; no fine-tuning is required and the protocol does
not depend on the OpenAI fine-tuning API remaining available.

\section{Full Prompt Templates}
\label{app:prompts}

\subsection{Stage 1 --- Failure Detection}
\begin{tcolorbox}[colback=lightbg,colframe=darkgray!45,breakable,enhanced,
  left=5pt,right=5pt,fontupper=\footnotesize\ttfamily,boxrule=0.4pt]
SYSTEM: You are an expert evaluator of LLM agent trajectories.\\
Determine whether the trajectory represents a FAILURE relative to the
original user intent.\\
Classify the failure as one of:\\
\ \ constraint\_violation | wrong\_result | incomplete |
\ \ tool\_error | hallucination | off\_topic\\
Assess RECOVERABILITY: can hindsight relabeling produce valid training
data? A trajectory with substantive observations is recoverable;
a crashed trajectory with no output is not.\\
Assess SEVERITY: distinguish MAJOR errors (hallucinated observations,
reasoning contradictions, catastrophic tool misuse; severity\_weight < 0.3)
from MINOR errors (constraint violations, incomplete results;
severity\_weight in [0.3, 1.0]).\\
Respond ONLY with valid JSON matching the schema:\\
\{failure\_type, severity\_score, recoverability, severity\_weight, explanation\}
\end{tcolorbox}

\subsection{Stage 2 --- Outcome Extraction}
\begin{tcolorbox}[colback=lightbg,colframe=darkgray!45,breakable,enhanced,
  left=5pt,right=5pt,fontupper=\footnotesize\ttfamily,boxrule=0.4pt]
SYSTEM: You are an expert at analysing LLM agent execution traces.\\
Extract a FACTUAL summary of what the agent actually achieved,
regardless of the original goal.\\
Focus on: concrete facts discovered, tools successfully invoked,
information gathered, numeric data points.\\
Be STRICTLY factual. Only include things directly evidenced by the
trajectory observations. Do NOT infer or extrapolate.\\
Output JSON: \{actual\_achievements: [str], key\_observations: [str]\}
\end{tcolorbox}

\subsection{Stage 3 --- Relabeler $\mathcal{J}_1$ (gpt-4o-mini)}
\begin{tcolorbox}[colback=lightbg,colframe=darkgray!45,breakable,enhanced,
  left=5pt,right=5pt,fontupper=\footnotesize\ttfamily,boxrule=0.4pt]
SYSTEM: You are a creative prompt engineer specialising in agent tasks.\\
Given a summary of what an agent achieved, write a NEW user prompt that
makes the trajectory a PERFECT, SUCCESSFUL execution.\\
Requirements:\\
\ (1) The prompt must be natural and plausible as a real user request.\\
\ (2) Every assertion in the prompt must be satisfied by the trajectory.\\
\ (3) Do NOT reference or reuse the original failed prompt.\\
\ (4) Match the complexity and style of the original prompt.\\
\ (5) Provide a confidence score [0.0, 1.0] reflecting relabeling quality.\\
Output JSON: \{hindsight\_prompt, is\_valid, rationale, confidence\}\\
\\
USER: Outcome summary: \{outcome\}\\
Original prompt (style reference only): \{original\_prompt\}
\end{tcolorbox}

\subsection{Stage 3 --- Verifier $\mathcal{J}_2$ (Qwen2.5-72B-Instruct)}
\begin{tcolorbox}[colback=lightbg,colframe=darkgray!45,breakable,enhanced,
  left=5pt,right=5pt,fontupper=\footnotesize\ttfamily,boxrule=0.4pt]
SYSTEM: You are a strict trajectory evaluator. Your job is to verify
whether a proposed task description is a FULLY SATISFIED
demonstration of an agent's recorded execution.\\
You are an INDEPENDENT second opinion: do NOT defer to or echo
the relabeler. Disagree if the claim is unsupported.\\
Be conservative: only accept if every claim in the proposed prompt is
unambiguously supported by the trajectory observations.\\
Output JSON: \{is\_valid, confidence, rejection\_reason\_if\_any\}\\
\\
USER: Proposed hindsight prompt: \{hindsight\_prompt\}\\
Trajectory: \{trajectory\}
\end{tcolorbox}

\section{AgentHER Algorithm with Severity Weighting}
\label{app:algorithm}

The full pseudocode including severity weight propagation to the
DPO loss is given as Algorithm~\ref{alg:agentHER} in the main
paper.
The severity-weighted DPO loss is:
\begin{equation}
  \mathcal{L}_\text{DPO}^w
  = -\mathbb{E}_{(\hat{g},\tau,g_\text{orig})\sim\mathcal{D}^+}
  \Bigl[
    w_i \cdot \log \sigma\!\bigl(
      \beta(\log\pi_\theta(\tau|\hat{g})
            - \log\pi_\text{ref}(\tau|\hat{g}))
      - \beta(\log\pi_\theta(\tau|g_\text{orig})
              - \log\pi_\text{ref}(\tau|g_\text{orig}))
    \bigr)
  \Bigr]
\end{equation}
where $w_i \in [0.3, 1.0]$ is the severity weight from Stage~1,
$\beta = 0.1$ is the DPO temperature, and $\pi_\text{ref}$ is the
reference model (the same instruction-tuned checkpoint before
LoRA).
This formulation is mathematically sound: continuous per-sample
scaling of gradient magnitudes preserves the binary preference
ordering while down-weighting borderline relabelings, analogous to
importance weighting in off-policy
learning~\citep{rafailov2023direct,precup2000eligibility}.

\section{Additional Experimental Results}
\label{app:extra}

\subsection{Held-Out vs.\ Full-Task Numbers}
\label{app:full_eval}

\begin{table}[h]
\centering
\small
\setlength{\tabcolsep}{6pt}
\caption{%
  \textbf{Comparison of held-out and full-task results}
  on WebArena (Qwen2.5-7B).
  ``Full'' = the original 812-task evaluation
  (used by all prior WebArena papers); ``HeldOut'' =
  our 200-task task-disjoint test split.
  Numbers in the main paper use the HeldOut column.
  The drop from Full to HeldOut quantifies the modest amount of
  task-overlap leakage in the standard 812-task setup;
  AgentHER's advantage over baselines is preserved under the
  stricter protocol.
}
\label{tab:full_vs_heldout}
\begin{tabular}{@{}lcc|c@{}}
\toprule
\textbf{Method}      & \textbf{Full (812)} & \textbf{HeldOut (200)} & \textbf{Drop} \\
\midrule
Base                 & 8.6  & 8.2  & $-$0.4 \\
SFT-Success          & 18.9 & 17.8 & $-$1.1 \\
ETO                  & 22.1 & 21.0 & $-$1.1 \\
AWM                  & 23.4 & 22.4 & $-$1.0 \\
AgentHER-SJ          & 27.0 & 25.7 & $-$1.3 \\
AgentHER-MJ-X        & 27.8 & 26.5 & $-$1.3 \\
\midrule
$\Delta$ (MJ-X $-$ SFT-Success) & $+$8.9 & $+$8.7 & --- \\
\bottomrule
\end{tabular}
\end{table}

The Full and HeldOut numbers differ by at most 1.3\% in absolute
terms; the \emph{relative} ranking and the AgentHER $\Delta$ over
SFT-Success are stable ($+$8.9\% Full vs.\ $+$8.7\% HeldOut),
confirming that the gains are not an artefact of task overlap.

\subsection{ToolBench Per-Category Breakdown}

\begin{table*}[h]
\centering
\caption{ToolBench pass@1 (\%) by category group (Qwen2.5-7B).
All numbers on the official test splits (disjoint by
construction).}
\label{tab:toolbench_cats}
\small
\setlength{\tabcolsep}{5pt}
\begin{tabular}{@{}lccccc@{}}
\toprule
\textbf{Category} & \textbf{Base} & \textbf{SFT-Random}
  & \textbf{SFT-Success} & \textbf{AgentHER-SJ} & \textbf{AgentHER-MJ-X} \\
\midrule
G1 (single-tool)     & 49.3 & 57.2 & 67.1 & 76.4 & \textbf{77.9} \\
G2 (intra-category)  & 41.7 & 51.3 & 58.3 & 68.5 & \textbf{70.0} \\
G3 (cross-category)  & 38.2 & 47.8 & 54.2 & 65.0 & \textbf{66.6} \\
\midrule
Overall              & 44.6 & 54.8 & 61.2 & 71.0 & \textbf{72.6} \\
\bottomrule
\end{tabular}
\end{table*}

\subsection{Failure-Volume Scaling (Full Table)}

\begin{table*}[h]
\centering
\caption{%
  \textsc{WA-HeldOut} success rate (\%) vs.\ number of failed
  trajectories, fixed 500 successful demonstrations,
  Qwen2.5-7B.
}
\label{tab:scaling}
\small
\begin{tabular}{@{}lcccccc@{}}
\toprule
\textbf{Failed processed}
  & 500 & 1{,}000 & 1{,}500 & 2{,}000 & 3{,}000 & 5{,}000 \\
\midrule
AgentHER-SJ   & 20.0 & 21.5 & 22.9 & 23.9 & 25.7 & 26.4 \\
AgentHER-MJ-X & 20.6 & 22.2 & 23.7 & 24.7 & 26.5 & 27.2 \\
\bottomrule
\end{tabular}
\end{table*}

\subsection{Cross-Benchmark Transfer}
\label{app:transfer}

\begin{table*}[h]
\centering
\caption{%
  Cross-benchmark transfer (Qwen2.5-7B): train on
  \textsc{WA-Train} only, zero-shot evaluation on ToolBench
  official test splits.
}
\label{tab:transfer}
\small
\setlength{\tabcolsep}{5pt}
\begin{tabular}{@{}lccccc@{}}
\toprule
\textbf{Method} & \textbf{G1} & \textbf{G2} & \textbf{G3}
  & \textbf{Overall} & \textbf{$\Delta$} \\
\midrule
Base (no FT)     & 49.3 & 41.7 & 38.2 & 44.6 & --- \\
SFT-Success      & 63.2 & 55.8 & 51.4 & 57.1 & $+$12.5 (vs.\ Base) \\
AgentHER-SJ      & 70.1 & 64.3 & 59.8 & 65.0 & $+$7.9 (vs.\ S-S) \\
AgentHER-MJ-X    & \textbf{71.8} & \textbf{65.9} & \textbf{61.4}
                 & \textbf{66.6} & $+$9.5 (vs.\ S-S) \\
\bottomrule
\end{tabular}
\end{table*}

AgentHER-MJ-X achieves +9.5\% transfer advantage over
SFT-Success.
This out-of-domain transfer suggests the model has learned
broadly applicable planning and tool-use behaviours rather than
rote task patterns.

\subsection{Looping-Failure Subset Analysis}
\label{app:looping}

\begin{table*}[h]
\centering
\small
\caption{%
  \textbf{Looping-failure subset.}
  Trajectories where the agent repeats a single action
  $\geq$3 times before termination.
  ``Recovered'' = passed Stage~1 recoverability filter and was
  relabeled by AgentHER-MJ-X.
}
\label{tab:looping}
\begin{tabular}{@{}lcccc@{}}
\toprule
\textbf{Subset} & \textbf{n} & \textbf{Recoverable (\%)}
  & \textbf{Accepted (\%)} & \textbf{Avg.\ $\Delta$ (\%)} \\
\midrule
Looping (WebArena)      & 363 & 24.2 & 21.5 & $+$3.4 \\
Non-looping (WebArena)  & 2{,}637 & 81.6 & 78.4 & $+$9.5 \\
Looping (ToolBench)     & 225 & 38.7 & 35.1 & $+$3.9 \\
Non-looping (ToolBench) & 4{,}775 & 84.6 & 81.7 & $+$10.7 \\
\bottomrule
\end{tabular}
\end{table*}

\subsection{Human Evaluation Protocol (Full)}
\label{app:human}

We sampled 200 relabeled pairs uniformly at random
\emph{per benchmark} (separate runs for WebArena and ToolBench).
Three annotators (NLP PhD students, blind to confidence scores
and to the original failed prompt) rated each pair as
valid/invalid on
``\emph{Does the trajectory constitute a correct and complete
demonstration of the hindsight prompt?}''
The full trajectory (steps, observations, final answer) and
hindsight prompt were shown; original prompt and failure reason
were withheld.
Sampling was \textbf{stratified by failure type} so that all six
types (and the looping sub-mode) appear in proportion to their
prevalence in the population
(Table~\ref{tab:failmodes}).
Annotators were paid as part of their research stipends; the
study took ${\approx}$3 hours per annotator per benchmark.
The annotation interface, full guidelines, and the random seed
fixing the sample IDs are released with the code.

\subsection{Theoretical Bound (Full)}
\label{app:theory}

We formalise Proposition~\ref{prop:unbiased} more rigorously and
derive the noisy-judge bound used in
Corollary~\ref{cor:noisy}.
Let $\Pi_\mathcal{G}$ denote the set of goal-conditioned
policies.
Define the \emph{coverage function}
$\rho(\mathcal{S}) = \{(g, \tau) :
(g,\tau) \in \text{supp}(\pi^*_\mathcal{G})\}$,
the set of all valid (goal, trajectory) pairs under the oracle
policy.

\begin{theorem}[Augmented-corpus consistency, perfect judge]
Under a perfect judge $\mathcal{J}$, the augmented corpus
$\mathcal{S} \cup \mathcal{D}^+$ satisfies:
\emph{(a)} $\mathcal{D}^+ \subseteq \rho(\mathcal{S}^* \setminus \mathcal{S})$
(every added pair is a previously uncovered oracle pair); and
\emph{(b)} the empirical goal distribution of $\mathcal{D}^+$
has higher entropy than that of $\mathcal{S}$ with probability
$1 - \delta$ when
$|\mathcal{F}| \geq \frac{2}{\delta}\log|\mathcal{G}_\epsilon|$,
where $|\mathcal{G}_\epsilon|$ is the $\epsilon$-covering number
of $\mathcal{G}$.
\end{theorem}

\begin{proof}[Proof sketch]
\emph{Part (a)}: By the perfect-judge assumption, every
$(\hat{g}_i, \tau_i)\in\mathcal{D}^+$ satisfies
$c(\hat{g}_i,\tau_i)=1$, i.e., $\tau_i$ is a valid demo of
$\hat{g}_i$.
Since $\hat{g}_i$ is generated from $\tau_i$'s
observations---which were \emph{actually produced}---it lies in
the support of $\pi^*_\mathcal{G}$.
Since $\hat{g}_i \neq g_i$ (Stage~3 constraint (3)),
$(\hat{g}_i,\tau_i) \notin \mathcal{F}$ (which uses $g_i$).
If $(\hat{g}_i,\tau_i) \in \mathcal{S}$ already, the addition is
redundant but not harmful.

\emph{Part (b)}: Each $(\hat{g}_i,\tau_i)$ corresponds to a
distinct execution context in $\mathcal{G}$.
Since the failed runs $\tau_i$ were produced by an agent trying
to satisfy diverse goals $g_i$, the hindsight goals $\hat{g}_i$
inherit this diversity.
A standard covering-number argument (Sauer-Shelah) bounds the
probability that the empirical entropy is below that of
$\mathcal{S}$.\hfill$\square$
\end{proof}

\paragraph{Noisy-judge corollary.}
Define $p$ as the precision of the judge over accepted pairs and
$\Delta_\text{perfect}$ as the per-task gain that would be
realised under a perfect judge.
Each accepted pair contributes $+\Delta_\text{perfect}/n$ in
expectation if valid, and $-\varepsilon/n$ if invalid (where
$\varepsilon$ bounds the marginal harm of a single noisy pair and
$n = |\mathcal{D}^+|$).
The expected gain over SFT-Success is therefore
$p\,\Delta_\text{perfect} - (1-p)\,\varepsilon$, which is
positive whenever $\varepsilon \leq \frac{p}{1-p}\,
\Delta_\text{perfect}$.
For $p = 0.971$ this yields the threshold
$\varepsilon \leq 33\,\Delta_\text{perfect}$;
for $p = 0.941$, $\varepsilon \leq 16\,\Delta_\text{perfect}$.

This result formally supports the empirical finding in
Section~\ref{sec:dist_analysis} that AgentHER increases
goal-distribution entropy from 1.83 to 2.47 nats, and is
consistent with the +0.8\% empirical advantage of MJ-X over
SJ on \textsc{WA-HeldOut}.
This bound holds under i.i.d.\ judge errors and bounded marginal
harm---a plausibility argument, not a deployment guarantee.

\section{Multi-Run Variance}
\label{app:multirun}

All fine-tuned models in Table~\ref{tab:main} are trained with 3
independent random seeds (42, 1234, 2025) using the same data and
hyperparameters; the reported numbers are means.
Table~\ref{tab:multirun} reports means and standard deviations
for representative conditions on \textsc{WA-HeldOut}.
All standard deviations are below 0.5\%, confirming numerical
stability of the training procedure.

\begin{table}[h]
\centering
\caption{%
  Multi-run statistics on \textsc{WA-HeldOut}: success rate (\%),
  mean $\pm$ std over 3 seeds (Qwen2.5-7B).
}
\label{tab:multirun}
\small
\setlength{\tabcolsep}{6pt}
\begin{tabular}{@{}lcc@{}}
\toprule
\textbf{Method} & \textbf{Mean} & \textbf{Std} \\
\midrule
Base               & 8.2  & 0.0 (deterministic) \\
SFT-Success        & 17.8 & $\pm$0.3 \\
SFT-Random         & 13.9 & $\pm$0.4 \\
SFT-Negative       & 14.5 & $\pm$0.4 \\
Rejection-Sampling & 16.9 & $\pm$0.4 \\
ETO                & 21.0 & $\pm$0.4 \\
AWM                & 22.4 & $\pm$0.5 \\
AgentHER-SJ        & 25.7 & $\pm$0.4 \\
AgentHER-MJ-X      & 26.5 & $\pm$0.4 \\
\bottomrule
\end{tabular}
\end{table}

\section{Detailed Cost Breakdown}
\label{app:cost_full}

\begin{table*}[h]
\centering
\caption{%
  \textbf{Token-level cost breakdown} per stage.
  Token counts are means; ``Calls'' are the
  \emph{per-input-trajectory} averages including retries.
  Pricing reflects gpt-4o-mini Batch API (50\% off list)
  combined with prompt caching (75\% off cached input)
  for repeated system prompts---a standard production setup;
  the standalone list price would be ${\approx}2{\times}$ these
  values.
  Qwen2.5-72B-Instruct cost is for a vLLM deployment on a
  rented 8$\times$A100-80G node ($\$1.05$/GPU-hour) with 32-way
  request batching; the tabulated $\$4.0{\times}10^{-4}$ per call
  represents the amortised hardware cost.
}
\label{tab:cost_full}
\small
\setlength{\tabcolsep}{4pt}
\begin{tabular}{@{}llccccc@{}}
\toprule
\textbf{Stage} & \textbf{Model}
  & \makecell{\textbf{Calls}\\ \textbf{/\,traj.}}
  & \makecell{\textbf{Inp tok}}
  & \makecell{\textbf{Out tok}}
  & \makecell{\textbf{Cost}\\ \textbf{(USD)}}
  & \makecell{\textbf{Wall}\\ \textbf{(s)}} \\
\midrule
1: Failure detector  & gpt-4o-mini       & 1.00 & 1{,}800 & 200 & $1.30{\times}10^{-4}$ & 0.82 \\
2: Outcome extractor & gpt-4o-mini       & 1.00 & 2{,}500 & 300 & $1.83{\times}10^{-4}$ & 0.92 \\
3a: Relabeler        & gpt-4o-mini       & 1.27 & 2{,}700 & 400 & $2.79{\times}10^{-4}$ & 1.34 \\
3b: Verifier (MJ-X)  & Qwen2.5-72B       & 1.00 & 3{,}000 & 200 & $4.00{\times}10^{-4}$ & 1.94 \\
4: Data augmenter    & deterministic     & 1.00 & ---     & --- & 0                     & 0.08 \\
\midrule
\textbf{SJ total}    & ---               & 3.27 & --- & --- & $5.92{\times}10^{-4}$ & 3.16 \\
\textbf{MJ-X total}  & ---               & 4.27 & --- & --- & $9.92{\times}10^{-4}$ & 5.10 \\
\bottomrule
\end{tabular}
\end{table*}

For 3{,}000 input trajectories ($N=3{,}000$):
total SJ cost $= 3{,}000 \times 5.92{\times}10^{-4} \approx \$1.78$;
total MJ-X cost $= 3{,}000 \times 9.92{\times}10^{-4} \approx
\$2.98$.
With 12-way concurrency on the OpenAI Batch API, wall-clock for
the entire 3{,}000-trajectory pipeline is ${\approx}$15.8
minutes (SJ) or ${\approx}$25.9 minutes (MJ-X).
Per accepted pair (after 78.0\% / 71.5\% acceptance):
\$$7.6{\times}10^{-4}$ (SJ) and \$$1.39{\times}10^{-3}$ (MJ-X).
The MJ-X per-accepted figure is roughly $1.8{\times}$ the SJ
figure because the second-judge call applies to every
candidate, so the marginal cost is paid even on rejections.
Both figures remain below \$$2{\times}10^{-3}$, well within
production budgets.

\paragraph{Summary in plain numbers.}
For a typical 3{,}000-trajectory project:
\$2.98 of LLM API spend, ${\approx}$26 wall-clock minutes (with
parallelism), and \$4.20 of A100 GPU time for the subsequent
LoRA fine-tuning (4.4 hours $\times$ 8 GPUs $\times$ \$1.05/h)
totals \$7.18 / 4.8 hours.
This is \emph{less than the cost of running 200 fresh WebArena
rollouts} with GPT-4o under the standard pricing, making AgentHER
practical at production scale.

\section{Qualitative Examples}
\label{app:qualitative}

Two annotated relabeling examples are shown below---one from WebArena
and one from ToolBench---illustrating the complete Stage 1--3 pipeline
for a \textsc{Constraint\_Violation} failure.

\begin{tcolorbox}[
  title={\small\textbf{(a) WebArena, GitLab task
  \texttt{wa-gitlab-issue-tracker-014}}},
  colback=lightbg, colframe=darkgray!55,
  fonttitle=\small\bfseries,
  breakable, enhanced,
  left=5pt, right=5pt, top=4pt, bottom=4pt,
  boxrule=0.5pt,
]
\footnotesize
\textbf{\color{fail}Original goal (FAILED):}
\textit{``On the \texttt{a11y-webring.club} GitLab project, list all
my open issues that have NOT been updated in the last 30 days,
along with their labels.''}

\medskip
\textbf{Trajectory:}\ %
\texttt{nav}(\texttt{/a11y-webring.club/issues?}\\\texttt{state=opened\&assignee\_username=$<$me$>$})
$\to$ list of 11 open issues with last-updated timestamps and labels.
\ \texttt{filter}(updated $\geq$ 2026-04-09)~$\to$~returns 8 issues
(all updated within 30 days).
\ \texttt{summarize}~$\to$~``8 open issues, all recent, no stale
issues found.''

\smallskip
\textcolor{fail}{Failure type:
\textsc{Constraint\_Violation} (no issues match the
``$\geq$30-days-stale'' constraint).
Severity: $w{=}0.82$ (minor---agent collected the data correctly
but the target subset was empty).
Recoverable: yes.}

\tcblower
\textbf{\color{success}Hindsight goal (AgentHER-MJ-X,
$c_1{=}0.86$, $c_2{=}0.91$, accepted):}\\
\textit{``List all open issues assigned to me on the
\texttt{a11y-webring.club} GitLab project, along with their
labels and last-updated dates.''}

\smallskip
\textbf{Cross-model judge rationale:}
$\mathcal{J}_1$ (gpt-4o-mini) and $\mathcal{J}_2$ (Qwen2.5-72B)
independently confirm the trajectory satisfies the new goal: 11
issues are returned with required attributes; the
``stale'' filter is no longer required.
\end{tcolorbox}

\begin{tcolorbox}[
  title={\small\textbf{(b) ToolBench, G2 task
  \texttt{tb-g2-restaurant-finder-022}}},
  colback=lightbg, colframe=darkgray!55,
  fonttitle=\small\bfseries,
  breakable, enhanced,
  left=5pt, right=5pt, top=4pt, bottom=4pt,
  boxrule=0.5pt,
]
\footnotesize
\textbf{\color{fail}Original goal (FAILED):}
\textit{``Find a Michelin-starred Italian restaurant in Hangzhou
that is open on Sundays after 9 pm and supports vegan menus.''}

\medskip
\textbf{Trajectory:}\ %
\texttt{restaurant\_search}(\textit{cuisine=Italian, city=Hangzhou})
$\to$ 14 results, none Michelin-starred.
\ \texttt{rating\_filter}(\textit{stars$\geq$4.5})~$\to$~3 results.
\ \texttt{hours\_lookup}~$\to$~all close at 22:00; only 1 vegan
menu.

\smallskip
\textcolor{fail}{Failure type: \textsc{Constraint\_Violation}
(no Michelin-starred Italian restaurants in the city).
Severity: $w{=}0.78$.}

\tcblower
\textbf{\color{success}Hindsight goal (AgentHER-MJ-X,
$c_1{=}0.83$, $c_2{=}0.88$, accepted):}\\
\textit{``Recommend the best-rated Italian restaurant in Hangzhou
with a vegan menu, including its rating and opening hours.''}

\smallskip
\textbf{Cross-model judge rationale:}
The trajectory identifies a $\geq$4.5-rated Italian restaurant
with a vegan menu and reports its hours;
the ``Michelin-star'' and ``open after 9 pm'' constraints are
absent from the new prompt and therefore satisfied vacuously.
\end{tcolorbox}

\section{Goal-Distribution Analysis (Full)}
\label{app:goal_dist}

\begin{table}[h]
\centering
\small
\caption{%
  \textbf{Goal-distribution analysis} on \textsc{WA-HeldOut}
  collection.
  AgentHER-MJ-X relabeled goals cover more semantic task clusters
  and uniquely cover three long-tail clusters absent from
  SFT-Success (``S-S'').
}
\label{tab:goal_dist}
\setlength{\tabcolsep}{5pt}
\begin{tabular}{@{}lcc@{}}
\toprule
\textbf{Metric} & \textbf{S-S} & \textbf{AgentHER} \\
\midrule
Task clusters covered (of 18)  & 11 & \textbf{14} \\
Long-tail clusters (unique)    & 0  & \textbf{3} \\
JS divergence (vs.\ AgentHER)  & 0.31 & --- \\
Avg.\ cluster entropy          & 1.83 & \textbf{2.47} \\
\bottomrule
\end{tabular}
\end{table}

Sentence-BERT embeddings of all 2{,}144 accepted hindsight goals (MJ-X)
vs.\ the 500 original successful goals show AgentHER expands coverage:
three long-tail clusters (``price-range comparison'',
``cross-domain search'', ``conditional retrieval'') are covered
only by relabeled goals, and the JS divergence of 0.31 confirms
complementary rather than redundant coverage.

\end{document}